\documentclass[letterpaper, 10 pt, conference]{ieeeconf}  
\IEEEoverridecommandlockouts                              
\usepackage{amsmath,amssymb,amsfonts}
\usepackage{algorithmic}
\usepackage{algorithm}
\usepackage{graphicx}
\usepackage{textcomp}
\usepackage{booktabs}
\usepackage{multirow}
\usepackage{bm}
\usepackage{cite}
\usepackage{url}
\usepackage{xcolor}
\usepackage{subcaption}
 
\newcommand{\R}{\mathbb{R}}
\newcommand{\btheta}{\bm{\theta}}
\newcommand{\bz}{\bm{z}}
\newcommand{\by}{\bm{y}}
\newcommand{\cW}{\mathcal{W}}
\newcommand{\cO}{\mathcal{O}}
\newcommand{\cR}{\mathcal{R}}
\newcommand{\psd}{\succeq}

\usepackage{amsmath,amssymb,amsfonts}
\usepackage{algorithmic}
\usepackage{algorithm}
\usepackage{graphicx}
\usepackage{textcomp}
\usepackage{xcolor}
\usepackage{bm}
\usepackage{mathtools}
\usepackage{booktabs}
\usepackage{multirow}
\usepackage{subcaption}
\usepackage{cite}
\usepackage{hyperref}

\DeclareMathOperator{\diag}{diag}

\newtheorem{theorem}{Theorem}
\newtheorem{proposition}{Proposition}
\newtheorem{remark}{Remark}
\newtheorem{problem}{Problem}

\usepackage{hyperref}   
\usepackage{cleveref}   
\usepackage{tikz}
\usepackage{booktabs}
\usepackage{url}
\usepackage{xcolor}
\usepackage{cite}
\usepackage[T1]{fontenc}

\title{\LARGE \bf 
Verified Task-Space Motion Planning Under Joint-Space Constraints
}

\author{Hanjiang Hu, Changliu Liu, Yebin Wang 
\thanks{H. Hu and C. Liu are with the Robotics Institute, Carnegie Mellon University. This work was done while H. Hu was a research intern at Mitsubishi Electric Research Laboratories (MERL), Cambridge, MA 02139, USA. (email: \tt\small hanjianghu@cmu.edu).}
\thanks{Y. Wang is with the MERL. (email: \tt\small yebinwang@ieee.org).}
}

\begin{document}

\maketitle
\thispagestyle{empty}
\pagestyle{empty}

\begin{abstract}
Reactive task-space planners such as Bug2 operate with fixed Cartesian
step sizes and are unaware of the manipulator's joint-angle limits.
When the Jacobian is poorly conditioned, even small Cartesian steps can
demand joint changes that exceed admissible bounds; clipping the joints
to their limits causes tracking drift and can prevent goal reaching
entirely.
We address this by computing, at each planning step, the largest
Cartesian hyperrectangle that is \emph{certifiably reachable} under
joint displacement bounds.
Using a second-order polynomial approximation of the inverse kinematics
and the S-procedure, we formulate a small semidefinite program whose
solution yields the certified half-width~$\lambda^\star$.
An equivalent bisection procedure exploiting the quadratic structure
solves the certification in sub-millisecond time.
Integrating this certificate with Bug2 yields a planner whose step
size adapts to local kinematic conditioning.
In a statistical evaluation over 94 adversarial scenarios spanning six
joint-limit settings, the SOS-verified planner achieves \emph{zero}
joint-limit violations with a 100\% goal-reaching rate, whereas a
standard Bug2 planner violates joint limits in 6--11\% of steps and
fails to reach the goal in up to 18\% of scenarios.
\end{abstract}

\maketitle
\thispagestyle{empty}
\pagestyle{empty}

\section{Introduction}
\label{sec:intro}
 
A fundamental challenge in robotic motion planning is the mismatch
between the \emph{task space}, where goals and obstacles are naturally
specified, and the \emph{joint space}, where the robot's physical
constraints reside.
High-level planners often generate Cartesian trajectories under the
implicit assumption that an inverse kinematics (IK) solution exists at
every waypoint, an assumption that can fail when joint-angle limits,
velocity bounds, or kinematic singularities are present
\cite{siciliano2009robotics, lavalle2006planning}.
In safety-critical applications such as industrial manipulation and
human--robot collaboration, violating these constraints is not merely
suboptimal but potentially hazardous.
 
Existing approaches to joint-aware task-space motion can be broadly
categorized as \emph{integrated} or \emph{hierarchical}.
Integrated methods plan directly in joint space while enforcing
task-space requirements: sampling-based planners project random
configurations onto task-constraint
manifolds~\cite{stilman2010global, kingston2018sampling}, and
trajectory optimizers jointly minimize cost subject to task and joint
constraints~\cite{schulman2013finding, liu2018convex}, which may fail to find a feasible solution if starting from a poor initial guess.
Hierarchical methods first plan a path in task space and then resolve
each waypoint to joint space via inverse kinematics, typically using the
Jacobian pseudoinverse or its damped
variants~\cite{siciliano2009robotics, colan2024variable}.
Reactive planners such as the Bug family~\cite{lumelsky1987path} are
a prominent instance of the hierarchical paradigm: they compute task-space
directions online and rely on local IK to execute each step.
We focus on the hierarchical setting because it is the natural framework
for reactive and online planning in partially known environments, where
task-space goals and obstacle geometry are discovered incrementally.
The central vulnerability of this paradigm is that a fixed Cartesian step
size that is safe at one configuration may violate joint limits at another,
due to the configuration-dependent amplification
$\|\Delta\btheta\| \leq \|\Delta\bz\|/\sigma_{\min}(J(\btheta))$.
Clipping the joints to their limits causes the end-effector to land at an
unintended position, producing cumulative tracking drift.

This paper addresses the question: \emph{given a robot's current
joint configuration and per-step joint displacement bounds, what is
the largest Cartesian region that is guaranteed to be reachable?}
If such a \emph{certified reachable set} can be computed efficiently
at each planning step, it can serve as a feasibility filter for any
task-space planner, ensuring that every commanded Cartesian
displacement admits a valid joint-space realization. Fig.~\ref{fig:overview} illustrates the idea: prior SOS-based
methods in robotics~\cite{tedrake2010lqr, majumdar2017funnel} certify
regions in \emph{state/joint space} where a closed-loop controller
keeps the system within a collision-free funnel, whereas our method
certifies a region in \emph{task space} that is kinematically reachable
from the current joint configuration under displacement bounds.
 
We answer this question using Sum-of-Squares (SOS) programming
\cite{parrilo2003semidefinite}.
Specifically, we construct a second-order polynomial approximation of
the inverse kinematics around the current configuration, and formulate
an SOS optimization that maximizes the half-width $\lambda^\star$ of a
Cartesian hyperrectangle subject to the constraint that the polynomial
IK maps every point in the hyperrectangle to joint displacements
within bounds.
The S-procedure \cite{yakubovich1971s, polik2007survey} provides a
computationally tractable sufficient condition via a small
semidefinite program.
Since the polynomial IK is quadratic, the resulting SOS relaxation
is exact (lossless) for our problem.
 
We integrate this certified reachable set with the Bug2 reactive
planner \cite{lumelsky1987path}, yielding a planner that \emph{adapts}
its Cartesian step size to the local kinematic conditioning.
In regions of high $\kappa(J)$, the certified set $\lambda^\star$
automatically shrinks and the planner takes smaller steps; in
well-conditioned regions, it takes larger steps.
Every step is guaranteed to satisfy joint limits by construction,
eliminating clipping-induced drift entirely. Note that while we demonstrate the approach with Bug2, the
certified reachable set is a general-purpose feasibility filter that
can be extended to any hierarchical task-space planner, to provide per-step
joint-feasibility guarantees. The contributions are listed below.
\begin{itemize}
  \item An SOS-based method to compute the maximal certified Cartesian
    reachable hyperrectangle from any joint configuration under
    per-step joint displacement bounds
    (Section~\ref{ssec:sos});
  \item Integration with the Bug2 reactive planner, yielding an
    adaptive-step algorithm with per-step joint-feasibility guarantees
    (Section~\ref{ssec:bug});
  \item A computationally efficient bisection implementation that
    exploits the quadratic structure of the polynomial IK, providing
    equivalent certificates to the full SOS program in sub-millisecond
    time (Section~\ref{ssec:efficient});
  \item A statistical evaluation across 94 adversarial scenarios and
    six joint-limit settings, demonstrating zero joint-limit violations
    and 100\% goal reaching where vanilla Bug2 fails
    (Section~\ref{sec:experiments}).
\end{itemize}

\begin{figure}[t]
\centering
\includegraphics[width=0.9\linewidth]{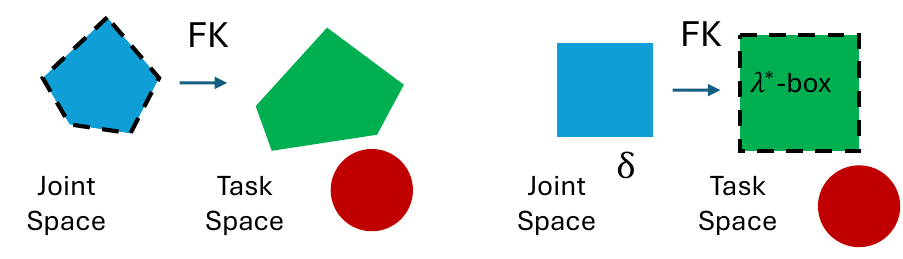}
\caption{
  \textbf{Left:} Prior methods certify
  collision-free regions in joint/state space (dashed-line) where a controller keeps
  the system within a safe funnel.
  \textbf{Right:} This work certifies a region in \emph{task space}
  (dashed-line $\lambda^\star$-box) that is guaranteed reachable from the
  current joint configuration under joint displacement bounds~$\bm{\delta}$.
  The certified box serves as a per-step feasibility filter for any
  task-space planner.}
\label{fig:overview}
\end{figure}

\section{Related Work}

\paragraph*{Integrated task-constrained planning}
Integrated methods plan directly in joint space while enforcing
task-space requirements.
Stilman~\cite{stilman2010global} introduced tangent-space sampling and
first-order retraction for RRT-based planning on constraint manifolds.
Kingston et al.~\cite{kingston2018sampling, kingston2019exploring}
unified these techniques in the IMACS framework, showing that various
constraint-adherence methods compose with standard sampling-based
planners.
Trajectory optimization approaches
\cite{schulman2013finding, liu2018convex, marcucci2023motion}
jointly optimize over the full trajectory subject to task, joint,
and collision constraints.
These methods are effective for offline or batch planning but do not
provide per-step kinematic feasibility certificates that a reactive
planner could use online.

\paragraph*{Hierarchical task-space planning}
Hierarchical methods first compute a task-space path and then resolve
each waypoint to joint space via IK.
The Jacobian pseudoinverse and its damped
variants~\cite{siciliano2009robotics, wampler2007manipulator,
nakamura1986inverse} are the standard IK mechanism; they improve
robustness near singularities but offer no formal guarantee that
joint limits are respected.
Variable step-size strategies~\cite{colan2024variable} adapt the
IK step heuristically for convergence speed but do not certify
joint-limit satisfaction.
Reactive planners---Bug algorithms~\cite{lumelsky1987path},
potential fields~\cite{khatib1986real}---are canonical instances
of the hierarchical paradigm, prized for online operation, but
oblivious to joint-space constraints.
Our work adds a certified feasibility layer to this paradigm, certifying
regions of \emph{task space} that are kinematically reachable under
joint constraints, a complementary notion suited to task-space
planning.
Fig.~\ref{fig:overview} contrasts the two settings.

 
 
\section{Problem Formulation}
\label{sec:problem}
 
We denote by $\|\cdot\|$ the Euclidean norm and by $\|\cdot\|_\infty$
the infinity norm.
For a symmetric matrix $M$ $\in \R^{d \times d}$
, we write $M \psd 0$ to indicate positive
semidefiniteness.
The set $\Sigma[y]$ denotes the cone of SOS polynomials
in variables $y$, i.e., $p \in \Sigma[y]$ if and only if $p(y) \geq 0$
for all $y$ and $p$ admits a decomposition $p = \sum_j q_j^2$ for some
polynomials $q_j$.
The singular values of a matrix $M$ are $\sigma_1 \geq \cdots \geq
\sigma_r > 0$, and the condition number is
$\kappa(M) = \sigma_1 / \sigma_r$.

\subsection{Planar Robot Model}
 
Consider an $n$-link planar revolute manipulator with link lengths
$\ell_1, \ldots, \ell_n > 0$.
The joint configuration is $\btheta = (\theta_1, \ldots, \theta_n)^\top
\in \R^n$, and the forward kinematics (FK) map
$\mathrm{FK}: \R^n \to \R^2$ 
maps joint angles to the end-effector position in the plane.
While we present the formulation for a planar manipulator with
$\mathrm{FK}: \R^n \to \R^2$ for notational clarity, the method
generalizes to spatial manipulators with
$\mathrm{FK}: \R^n \to \R^m$ ($m \leq 6$): the polynomial IK
approximation \eqref{eq:poly_ik} extends naturally to $m$-dimensional
task space, and the SOS certificate \eqref{eq:s_proc} involves
$(m{+}1) \times (m{+}1)$ matrices (see Remark~\ref{rem:generalize}).
The planar case is representative of SCARA-type manipulators,
planar parallel mechanisms, and the in-plane motion of spatial
arms constrained to a vertical plane (e.g., sagittal-plane
reaching).
The forward kinematics is
\begin{equation}\label{eq:fk}
  \bz = \mathrm{FK}(\btheta)
  = \begin{pmatrix}
    \sum_{i=1}^n \ell_i \cos\theta_i \\[3pt]
    \sum_{i=1}^n \ell_i \sin\theta_i
  \end{pmatrix},
\end{equation}
where $\bz = (z_1, z_2)^\top \in \R^2$ is the end-effector position in
the task (Cartesian) space $\cW \subset \R^2$.
The manipulator Jacobian is
\begin{equation}\label{eq:jacobian}
  J(\btheta) = \frac{\partial\, \mathrm{FK}}{\partial \btheta}
  \in \R^{2 \times n},
  \quad
  J_{ki} =
  \begin{cases}
    -\ell_i \sin\theta_i, & k=1, \\[2pt]
     \ell_i \cos\theta_i, & k=2.
  \end{cases}
\end{equation}

\subsection{Joint-Limit Constraints}
 
In many applications, the joint displacement per planning step is
bounded.
These bounds arise from physical joint-angle limits of the actuators,
velocity or torque limits (which, for a fixed planning rate $\Delta t$,
translate to maximum angular displacements
$\delta_i = \dot{\theta}_{i,\max} \Delta t$), or
safety constraints imposed by the operating environment (e.g.,
collaborative workspaces where large joint motions are restricted to
limit the robot's swept volume).
Formally, at each planning step the change in joint angles from the
current configuration $\btheta^{(t)}$ must satisfy
\begin{equation}\label{eq:joint_lim}
  |\theta_i^{(t+1)} - \theta_i^{(t)}| \leq \delta_i,
  \quad i = 1, \ldots, n,
\end{equation}
where $\bm{\delta} = (\delta_1, \ldots, \delta_n)^\top > \bm{0}$ are
given bounds.
We denote the admissible joint displacement set at configuration
$\btheta^{(t)}$ by
\begin{equation}\label{eq:Theta_set}
  \Theta(\btheta^{(t)}, \bm{\delta})
  \coloneqq
  \bigl\{
    \btheta \in \R^n :
    |\theta_i - \theta_i^{(t)}| \leq \delta_i, \; i=1,\ldots,n
  \bigr\}.
\end{equation}
When all bounds are equal ($\delta_i = \delta$ for all $i$),
this reduces to $\|\btheta - \btheta^{(t)}\|_\infty \leq \delta$.

\subsection{Task-Space Planning Problem}
 
Let $\bz_{\mathrm{start}} = \mathrm{FK}(\btheta^{(0)})$ be the initial
end-effector position and $\bz_{\mathrm{goal}} \in \cW$ the desired
target.
The workspace contains a set of obstacles
$\cO = \bigcup_{j=1}^{N_{\mathrm{obs}}} \cO_j \subset \cW$.
 
\begin{problem}[Joint-Limited Task-Space Planning]\label{prob:main}
Find a sequence of joint configurations
$\{\btheta^{(t)}\}_{t=0}^{T}$ such that:
\begin{enumerate}
  \item \textbf{Goal reaching:}
    $\|\mathrm{FK}(\btheta^{(T)}) - \bz_{\mathrm{goal}}\| \leq
    \epsilon_{\mathrm{tol}}$;
  \item \textbf{Joint feasibility:}
    $\btheta^{(t+1)} \in \Theta(\btheta^{(t)}, \bm{\delta})$
    for all $t = 0, \ldots, T{-}1$;
  \item \textbf{Obstacle avoidance:}
    $\mathrm{FK}(\btheta^{(t)}) \notin \cO$ for all $t$
    (end-effector collision avoidance; full-body collision
    checking can be incorporated by inflating the obstacles to account
    for the link geometry, which is standard for
    low-DOF planar arms~\cite{lavalle2006planning}).
\end{enumerate}
\end{problem}
 
\subsection{Challenge: Kinematic Coupling
}
 
A standard approach is to plan a Cartesian path and convert each
waypoint to a joint-space configuration via the pseudoinverse mapping
$\Delta\btheta = J(\btheta)^\dagger \Delta\bz$, where each Cartesian
step $\Delta\bz$ is chosen with a fixed magnitude~$s$.
The difficulty is that
\begin{equation}\label{eq:amplification}
  \|\Delta\btheta\|
  \leq \|J^\dagger\| \cdot \|\Delta\bz\|
  = \frac{\|\Delta\bz\|}{\sigma_{\min}(J)},
\end{equation}
where $\sigma_{\min}(J)$ is the smallest singular value.
The condition number $\kappa(J) = \sigma_{\max}/\sigma_{\min}$ varies
across the workspace: in high-$\kappa$ regions, even small Cartesian
displacements can require joint changes exceeding $\bm{\delta}$.
A fixed Cartesian step size $s = \delta / \kappa(\btheta^{(0)})$
that is safe at the initial configuration may violate
\eqref{eq:joint_lim} when $\kappa$ increases along the path.

A natural idea is to update the step size at each step using the
current condition number, i.e., $s^{(t)} = \delta / \kappa(\btheta^{(t)})$.
However, this heuristic has two limitations.
First, it uses only the \emph{worst-case singular-value bound}
\eqref{eq:amplification}, which can be overly conservative since it
does not account for the direction of $\Delta\bz$ relative to the
singular vectors of $J$.
Second, it relies on a linear (first-order) IK model, ignoring the
quadratic and higher-order terms that become significant in
poorly-conditioned regions.
Variable step-size strategies for iterative IK have been explored
in~\cite{colan2024variable}, but these adapt the step heuristically
for convergence speed rather than providing formal joint-limit
certificates.
In contrast, our SOS-based approach computes $\lambda^\star$ by
certifying the \emph{exact} quadratic polynomial IK over the full
Cartesian box, providing a tight and formally guaranteed bound.
Clipping the joint displacement to $\pm\bm{\delta}$ produces a
\emph{different} end-effector position than intended, causing
cumulative Cartesian tracking drift.
 
\section{Method}
\label{sec:method}
 
Our approach has three components:
(i)~a second-order polynomial approximation of the local inverse
kinematics,
(ii)~an SOS-based certification of the maximal Cartesian reachable set
under joint limits, and
(iii)~integration with the Bug2 reactive planner. Components (i)--(ii) address the core limitation of existing
task-space planners: they provide a \emph{certified} answer to the
question ``how far can the end-effector move while guaranteeing
joint feasibility?''
The polynomial approximation (i) captures the nonlinear
coupling between task and joint space that first-order
(Jacobian-only) methods miss, while the SOS certificate (ii)
provides a formal guarantee over the \emph{entire} Cartesian box
rather than a single-direction worst-case bound.
Component (iii) demonstrates the utility of these certificates in a
concrete planning setting.
 
\subsection{Polynomial Inverse Kinematics}
\label{ssec:poly_ik}
 
At the current configuration $\btheta_0 \coloneqq \btheta^{(t)}$ with
end-effector position $\bz_0 = \mathrm{FK}(\btheta_0)$, we seek a
polynomial map $\hat{\btheta}: \R^2 \to \R^n$ such that
$\hat{\btheta}(\Delta\bz) \approx \mathrm{IK}(\bz_0 + \Delta\bz)$ for
small displacements $\Delta\bz = (\Delta z_1, \Delta z_2)^\top$.
 
\subsubsection{First-Order Term}
 
The linear approximation uses the Moore--Penrose pseudoinverse:
\begin{equation}\label{eq:linear_ik}
  \hat{\btheta}^{(1)}(\Delta\bz)
  = \btheta_0 + A \, \Delta\bz,
  \quad A \coloneqq J(\btheta_0)^\dagger \in \R^{n \times 2}.
\end{equation}
 
\subsubsection{Second-Order Term}
 
To capture nonlinear IK behavior, we augment with quadratic terms.
Define the correction matrix for joint $i$:
\begin{equation}\label{eq:B_coeffs}
  B_i = \begin{pmatrix} b_{i,11} & b_{i,12}/2 \\
  b_{i,12}/2 & b_{i,22} \end{pmatrix}
  \in \R^{2 \times 2},
\end{equation}
where the coefficients capture how $J^\dagger$ varies along the IK
manifold.
These are computed via finite differences of $J^\dagger$ evaluated at
perturbed configurations $\btheta_0 + A \bm{e}_k h$ for small $h > 0$:
\begin{align}
  b_{i,11} &= \tfrac{1}{2h}
    \bigl( [J^\dagger(\btheta_0 {+} A\bm{e}_1 h)]_{i1}
           - A_{i1} \bigr), \label{eq:b11} \\
  b_{i,22} &= \tfrac{1}{2h}
    \bigl( [J^\dagger(\btheta_0 {+} A\bm{e}_2 h)]_{i2}
           - A_{i2} \bigr), \label{eq:b22} \\
  b_{i,12} &= \tfrac{1}{h}
    \bigl( [J^\dagger(\btheta_0 {+} A\bm{e}_1 h)]_{i2}
           - A_{i2} \bigr). \label{eq:b12}
\end{align}
 
The full second-order polynomial IK for joint $i$ is then
\begin{multline}\label{eq:poly_ik}
  \hat{\theta}_i(\Delta\bz)
  = \theta_{0,i}
    + A_{i1}\Delta z_1 + A_{i2}\Delta z_2 \\
    + b_{i,11} \Delta z_1^2
    + b_{i,12} \Delta z_1 \Delta z_2
    + b_{i,22} \Delta z_2^2.
\end{multline}
 
\subsubsection{Approximation Error Bound}
 
Let $\varepsilon(\rho)$ denote the worst-case approximation error over
the Cartesian box $[-\rho,\rho]^2$:
\begin{equation}\label{eq:approx_err}
  \varepsilon(\rho) \coloneqq
  \max_{\|\Delta\bz\|_\infty \leq \rho}
  \bigl\|
    \mathrm{FK}\bigl(\hat{\btheta}(\Delta\bz)\bigr)
    - (\bz_0 + \Delta\bz)
  \bigr\|.
\end{equation}
This is evaluated numerically on a dense grid (we use a $7 \times 7$ grid, requiring 49 forward-kinematics
evaluations per step---negligible compared to the bisection cost).
The effective joint bounds used for certification are then
\begin{equation}\label{eq:delta_eff}
  \delta_i^{\mathrm{eff}} \coloneqq \delta_i - \varepsilon(\rho),
  \quad i = 1, \ldots, n.
\end{equation}
We clarify the role of $\varepsilon(\rho)$: it bounds the
\emph{task-space} error of the polynomial IK model, i.e., the gap
between the target Cartesian position and the actual position
achieved by the polynomial IK joint angles.
Since $\varepsilon(\rho)$ also bounds the maximum \emph{joint-space}
error $\|\hat{\btheta}(\Delta\bz) - \btheta^{\mathrm{true}}\|_\infty$
(via the Jacobian
), subtracting it from the joint bounds ensures
that the true IK solution---not just the polynomial
approximation---respects~\eqref{eq:joint_lim}. We assume the Jacobian is non-singular to avoid finding $\rho$ which gives a small $\varepsilon$.
More precisely, if the polynomial IK predicts a joint displacement
$\Delta\hat{\theta}_i$ and the true displacement is
$\Delta\theta_i^{\mathrm{true}}$, then
$|\Delta\hat{\theta}_i - \Delta\theta_i^{\mathrm{true}}| \leq
C \cdot \varepsilon(\rho)$ where $C$ depends on the local Jacobian;
the reduction $\delta_i^{\mathrm{eff}} = \delta_i - \varepsilon(\rho)$
conservatively accounts for this gap.
If $\delta_i^{\mathrm{eff}} \leq 0$ for any $i$, the polynomial
approximation is too coarse and $\rho$ must be reduced.
 
\subsection{SOS-Based Reachable Set Certification}
\label{ssec:sos}
 
We now formulate the core certification problem: given the polynomial IK
\eqref{eq:poly_ik} and effective joint bounds
$\bm{\delta}^{\mathrm{eff}}$, find the largest $\lambda > 0$ such that
\begin{equation}\label{eq:cert_goal}
  \forall \, \Delta\bz \in [-\lambda, \lambda]^2:
  \
  |\hat{\theta}_i(\Delta\bz) - \theta_{0,i}| \leq
  \delta_i^{\mathrm{eff}},
  \
  i = 1, \ldots, n.
\end{equation}
 
\subsubsection{Quadratic Form Representation}
 
Define the monomial vector $\by = (1, \Delta z_1, \Delta z_2)^\top$.
The polynomial IK displacement for joint $i$ can be written as
\begin{equation}\label{eq:quad_form}
  \Delta\hat{\theta}_i(\Delta\bz)
  = \hat{\theta}_i(\Delta\bz) - \theta_{0,i}
  = \by^\top Q_i \, \by,
\end{equation}
where
\begin{equation}\label{eq:Qi}
  Q_i =
  \begin{pmatrix}
    0            & A_{i1}/2      & A_{i2}/2 \\
    A_{i1}/2     & b_{i,11}      & b_{i,12}/2 \\
    A_{i2}/2     & b_{i,12}/2    & b_{i,22}
  \end{pmatrix}
  \in \mathbb{S}^3.
\end{equation}
 
The constraint \eqref{eq:cert_goal} decomposes into $2n$ scalar
polynomial nonpositivity conditions.
For each joint $i$ and sign $\sigma \in \{+1, -1\}$:
\begin{equation}\label{eq:p_constraint}
  p_{i,\sigma}(\Delta\bz)
  \coloneqq
  \sigma \cdot \by^\top Q_i \, \by - \delta_i^{\mathrm{eff}}
  \leq 0,
  \quad
  \forall\, \Delta\bz \in [-\lambda, \lambda]^2.
\end{equation}
 
\subsubsection{S-Procedure Relaxation}
 
The domain constraint $\Delta\bz \in [-\lambda,\lambda]^2$ is encoded
via two nonnegative polynomials:
\begin{align}
  g_1(\Delta\bz) &= \lambda^2 - \Delta z_1^2 \geq 0, \label{eq:g1} \\
  g_2(\Delta\bz) &= \lambda^2 - \Delta z_2^2 \geq 0. \label{eq:g2}
\end{align}
In the $\by$-basis, these are
$g_k(\Delta\bz) = \by^\top G_k(\lambda)\by$ with
\begin{equation}\label{eq:G_matrices}
  G_1 \!=\!
  \begin{pmatrix}
    \lambda^2 \!&\! 0 \!&\! 0 \\
    0 \!&\! -1 \!&\! 0 \\
    0 \!&\! 0  \!&\! 0
  \end{pmatrix}\!,
  \;
  G_2 \!=\!
  \begin{pmatrix}
    \lambda^2 \!&\! 0 \!&\! 0 \\
    0 \!&\! 0  \!&\! 0 \\
    0 \!&\! 0  \!&\! -1
  \end{pmatrix}\!.
\end{equation}
 
By the generalized S-procedure \cite{yakubovich1971s, polik2007survey}, a
sufficient condition for \eqref{eq:p_constraint} is:
 
\begin{proposition}[S-Procedure Certificate]\label{prop:s_proc}
The constraint $p_{i,\sigma}(\Delta\bz) \leq 0$ for all
$\Delta\bz \in [-\lambda,\lambda]^2$ holds if there exist multipliers
$c_1, c_2 \geq 0$ such that
\begin{equation}\label{eq:s_proc}
  -\sigma \cdot Q_i
  + \delta_i^{\mathrm{eff}} \cdot E_{11}
  + c_1 \, G_1(\lambda)
  + c_2 \, G_2(\lambda)
  \psd 0,
\end{equation}
where $E_{11} = \diag(1, 0, 0)$.
\end{proposition}
 
\begin{proof}
Multiplying $g_k \geq 0$ by $c_k \geq 0$ and summing gives
$\sum_k c_k g_k(\Delta\bz) \geq 0$ on the domain.
Adding to $-p_{i,\sigma}(\Delta\bz) = -\sigma \cdot \by^\top Q_i \by
+ \delta_i^{\mathrm{eff}}$ yields
$\by^\top S \, \by \geq 0$ with $S$ as in \eqref{eq:s_proc}.
If $S \psd 0$, then $-p_{i,\sigma}(\Delta\bz) \geq 0$, i.e.,
$p_{i,\sigma}(\Delta\bz) \leq 0$, as required.
\end{proof}
 
\subsubsection{Optimization Problem}
 
The maximal certified Cartesian half-width $\lambda^\star$ is the
solution of:
\begin{subequations}\label{eq:opt}
\begin{align}
    \lambda^\star = \max_{\lambda, \, c_1^{(i,\sigma)},
    \, c_2^{(i,\sigma)}}
    \;\; & \lambda \label{eq:opt_obj} \\
    \text{s.t.} \;\;
    & S_{i,\sigma}(\lambda, c_1, c_2) \psd 0,
    \; \forall\, i, \sigma, \label{eq:opt_psd} \\
    & c_1^{(i,\sigma)}, c_2^{(i,\sigma)} \geq 0, \label{eq:opt_c} \\
    & 0 \leq \lambda \leq \lambda_{\max}, \label{eq:opt_lam}
\end{align}
\end{subequations}
where
\begin{equation}\label{eq:S_matrix}
  S_{i,\sigma} \coloneqq
  -\sigma \, Q_i
  + \delta_i^{\mathrm{eff}} E_{11}
  + c_1 G_1(\lambda)
  + c_2 G_2(\lambda).
\end{equation}
 
This is a bilinear SDP in $(\lambda, c_1, c_2)$ due to the
$c_k \lambda^2$ terms in $c_k G_k(\lambda)$.
For a fixed $\lambda$, \eqref{eq:opt} reduces to a standard SDP
feasibility problem.
We solve \eqref{eq:opt} via nonlinear programming with the
positive semidefiniteness condition enforced through the Sylvester
criterion (all leading principal minors $\geq 0$):
\begin{align}
  [S]_{11} &\geq 0, \label{eq:syl1} \\
  [S]_{11}[S]_{22} - [S]_{12}^2 &\geq 0, \label{eq:syl2} \\
  \det(S) &\geq 0. \label{eq:syl3}
\end{align}
 
\begin{theorem}[Certified Reachable Set]\label{thm:main}
If $\lambda^\star > 0$ is a feasible solution of \eqref{eq:opt} at
configuration $\btheta_0$, then for every
$\Delta\bz \in [-\lambda^\star, \lambda^\star]^2$, the polynomial IK
\eqref{eq:poly_ik} produces joint displacements satisfying
$|\Delta\hat{\theta}_i| \leq \delta_i^{\mathrm{eff}}
\leq \delta_i$ for all $i$.
Consequently, the Cartesian box
\begin{equation}\label{eq:reachable_set}
  \cR(\btheta_0) \coloneqq
  \bigl\{
    \bz_0 + \Delta\bz : \|\Delta\bz\|_\infty \leq \lambda^\star
  \bigr\}
\end{equation}
is a verified inner approximation of the set of end-effector positions
reachable from $\btheta_0$ under joint limits $\bm{\delta}$.
\end{theorem}
 
\begin{proof}
By Proposition~\ref{prop:s_proc}, feasibility of \eqref{eq:opt}
implies $|\by^\top Q_i \by| \leq \delta_i^{\mathrm{eff}}$ for all
$\Delta\bz \in [-\lambda^\star, \lambda^\star]^2$ and all $i$.
By \eqref{eq:quad_form}, this gives
$|\Delta\hat{\theta}_i(\Delta\bz)| \leq \delta_i^{\mathrm{eff}}$.
Since $\delta_i^{\mathrm{eff}} = \delta_i - \varepsilon(\rho) \leq
\delta_i$, the true IK displacement also satisfies the original bounds
\eqref{eq:joint_lim} up to the approximation error already subtracted.
\end{proof}
 
\begin{remark}[Exactness for Quadratics]
Since the constraint polynomial $p_{i,\sigma}$ in
\eqref{eq:p_constraint} and the domain polynomials $g_1, g_2$ in
\eqref{eq:g1}--\eqref{eq:g2} are all quadratic, the S-procedure
relaxation is exact (lossless) for quadratic polynomials over
quadratic domains \cite{yakubovich1971s}.
\end{remark}

 \begin{remark}[Generalization to Higher Dimensions]\label{rem:generalize}
For a spatial manipulator with $\mathrm{FK}: \R^n \to \R^m$
($m \leq 6$), the monomial vector becomes
$\by = (1, \Delta z_1, \ldots, \Delta z_m)^\top \in \R^{m+1}$,
the polynomial IK matrix $Q_i \in \mathbb{S}^{m+1}$, and the
domain is encoded by $m$ box constraints $g_k = \lambda^2 - \Delta z_k^2 \geq 0$.
The S-procedure certificate is an $(m{+}1) \times (m{+}1)$ PSD
condition with $m$ multipliers.
The bisection procedure generalizes directly: the grid becomes
$m$-dimensional, and the Pataki--Barvinok rank bound still ensures
SDP exactness provided the number of affine constraints ($m{+}1$)
satisfies $r(r{+}1)/2 \leq m{+}1$, i.e., $r \leq 1$.
For $m \leq 6$, this is a small SDP and remains computationally
tractable.
\end{remark}

\subsection{Integration with Bug2 Planner 
}
\label{ssec:bug}

The Bug2 algorithm \cite{lumelsky1987path} alternates between two modes:
\emph{go-to-goal} (GTG), moving directly toward the target, and
\emph{boundary-follow} (BF), circumnavigating an obstacle until the
start--goal line is crossed closer to the goal.
The standard Bug2 uses a fixed Cartesian step and is unaware of joint
limits.
We modify it with an adaptive step size driven by the certified
reachable set.

We emphasize that the adaptive step-size mechanism and the
polynomial IK update are not specific to Bug2: they can be used
with any task-space planner that produces a desired Cartesian
displacement direction $\bm{d}^{(t)}$ at each step.
We choose Bug2 because (a)~its fixed-step structure makes the
failure mode under joint limits most transparent, and (b)~its
completeness guarantee (for point robots in the plane) transfers
to our setting whenever $\lambda^\star > 0$ everywhere along the
path.

\subsubsection{Guaranteed Properties}
 
Before describing the algorithm, we state its key properties.
 
\begin{proposition}[Per-Step Feasibility]\label{prop:feasibility}
If $\lambda^\star > 0$ at step $t$, then
$\btheta^{(t+1)} \in \Theta(\btheta^{(t)}, \bm{\delta})$.
\end{proposition}
\begin{proof}
Immediate from Theorem~\ref{thm:main} and the construction
\eqref{eq:step} ensuring
$\Delta\bz^{(t)} \in [-\lambda^\star, \lambda^\star]^2$.
\end{proof}
 
 
\begin{remark}[Monotone Approach in GTG and Completeness]
In GTG mode with no obstacle intersection, the distance to goal
decreases by at least
$\alpha\lambda^\star \cos\phi$ per step, where $\phi$ is the angle
between the goal direction and the nearest axis-aligned direction
in the $\lambda^\star$-box.
Bug2 is complete for simply-connected obstacles in the plane
\cite{lumelsky1987path}.
Our modification preserves this provided $\lambda^\star > 0$ at every
reachable configuration.
If $\lambda^\star = 0$ (kinematic singularity or exhausted joint
limits), the planner reports infeasibility.
\end{remark}

\subsubsection{Adaptive Step Size}
 
At each step $t$, compute $\lambda^\star(\btheta^{(t)})$ from
\eqref{eq:opt}.
The Cartesian displacement is then
\begin{equation}\label{eq:step}
  \Delta\bz^{(t)}
  = \mathrm{clip}\!\left(
    \alpha \lambda^\star \cdot \bm{d}^{(t)},\;
    -\lambda^\star,\; \lambda^\star
  \right),
\end{equation}
where $\bm{d}^{(t)}$ is the unit direction (toward goal in GTG,
tangential in BF), $\alpha \in (0,1)$ is a conservatism factor,
and $\mathrm{clip}$ denotes component-wise clamping.
 
\subsubsection{Joint Update via Polynomial IK}
 
The joint update uses the second-order polynomial \eqref{eq:poly_ik}:
\begin{equation}\label{eq:joint_update}
  \btheta^{(t+1)} = \hat{\btheta}\!\left(\Delta\bz^{(t)}\right).
\end{equation}
By Theorem~\ref{thm:main}, since
$\Delta\bz^{(t)} \in [-\lambda^\star, \lambda^\star]^2$, the joint
displacement satisfies \eqref{eq:joint_lim} automatically.
 
The complete algorithm is given in Algorithm~\ref{alg:sos_bug}.
 
\begin{algorithm}[t]
\caption{SOS-Verified Bug2 Planner}
\label{alg:sos_bug}
\begin{algorithmic}[1]
\REQUIRE $\btheta^{(0)}$, $\bz_{\mathrm{goal}}$, $\bm{\delta}$,
         $\cO$, $\alpha$, $\epsilon_{\mathrm{tol}}$, $\rho$
\STATE $\mathrm{mode} \leftarrow \text{GTG}$;
       $\bz^{(0)} \leftarrow \mathrm{FK}(\btheta^{(0)})$
\FOR{$t = 0, 1, 2, \ldots$}
  \IF{$\|\bz^{(t)} - \bz_{\mathrm{goal}}\| < \epsilon_{\mathrm{tol}}$}
    \RETURN success
  \ENDIF
  \STATE Compute $A, B$ via
    \eqref{eq:linear_ik}--\eqref{eq:b12}
  \STATE Compute $\varepsilon(\rho)$ via \eqref{eq:approx_err};
         set $\bm{\delta}^{\mathrm{eff}}$ via \eqref{eq:delta_eff}
  \STATE Solve \eqref{eq:opt} $\Rightarrow \lambda^\star$
  \IF{$\lambda^\star < \epsilon_{\mathrm{min}}$}
    \STATE Retry with smaller $\rho$
  \ENDIF
  \STATE Compute $\bm{d}^{(t)}$ (GTG or BF)
  \STATE $\Delta\bz^{(t)} \leftarrow
         \mathrm{clip}(\alpha\lambda^\star\bm{d}^{(t)},\,
         {-}\lambda^\star,\, \lambda^\star)$
  \IF{$\bz^{(t)} + \Delta\bz^{(t)} \in \cO$}
    \STATE Switch to BF; recompute $\Delta\bz^{(t)}$
  \ENDIF
  \STATE $\btheta^{(t+1)} \leftarrow \hat{\btheta}(\Delta\bz^{(t)})$
         via \eqref{eq:poly_ik}
  \STATE $\bz^{(t+1)} \leftarrow \mathrm{FK}(\btheta^{(t+1)})$
\ENDFOR
\end{algorithmic}
\end{algorithm}
 
 
 
 
 
\subsection{Efficient Implementation}
\label{ssec:efficient}
 
While the SOS program \eqref{eq:opt} provides the theoretical
foundation, solving the bilinear SDP at each planning step via a
general-purpose NLP solver (e.g., Ipopt \cite{wachter2006implementation}) is
computationally expensive.
 
We exploit the fact that the S-procedure matrix
$S_{i,\sigma} \in \mathbb{S}^3$ (cf.~\eqref{eq:S_matrix}) is only
$3 \times 3$, so the certification problem has very low
dimensionality.
For a degree-2 polynomial over a box domain, the maximum of
$|\Delta\hat{\theta}_i(\Delta\bz)|$ over
$[-\lambda,\lambda]^2$ can be evaluated exactly on a sufficiently dense
grid, since the optimizer lies at a vertex, edge extremum, or the
unique interior critical point.
This yields a bisection procedure: for each joint $i$, binary search
on $\lambda$ using the grid maximum as the feasibility oracle.
With 50 bisection iterations and a $21 \times 21$ grid, the total
cost is $n \times 50 \times 441 \approx 66{,}000$ floating-point
operations per planning step, completing in sub-millisecond time.
 
\begin{proposition}[Equivalence]\label{prop:bisection}
For degree-2 polynomials over boxes, the bisection procedure computes
the same $\lambda^\star$ as the SOS program \eqref{eq:opt}, since the
SOS relaxation is exact for quadratic polynomials over quadratic
domains.
\end{proposition}

\begin{proof}
Fix joint $i$ and sign $\sigma \in \{+1,-1\}$.
For a given $\lambda$, the question of whether
$\sigma \cdot \by^\top Q_i \by \leq \delta_i^{\mathrm{eff}}$
for all $\Delta\bz \in [-\lambda,\lambda]^2$
can be lifted to a semidefinite program by introducing
$Y = \by\by^\top \in \mathbb{S}^3$ and relaxing $Y = \by\by^\top$
to $Y \psd 0$:
\begin{align*}
  p^\star = \max \;\; & \mathrm{tr}(\sigma Q_i \cdot Y) \\
  \text{s.t.} \;\; & Y_{00} = 1, \;
    Y_{11} \leq \lambda^2, \;
    Y_{22} \leq \lambda^2, \;
    Y \psd 0.
\end{align*}
The Lagrangian dual of this SDP, with multipliers $\nu$ for
$Y_{00}=1$ and $\mu_1, \mu_2 \geq 0$ for the box constraints,
yields precisely the S-procedure certificate \eqref{eq:s_proc}
upon identifying $\nu = \delta_i^{\mathrm{eff}}$ and
$\mu_k = c_k$.

\emph{Strong duality:}
Slater's condition holds for the primal SDP
(e.g., $Y = \frac{1}{2}\lambda^2 I$ is strictly feasible for
sufficiently small scaling), so the primal and dual optima
coincide.

\emph{SDP exactness:}
The primal has a $3 \times 3$ matrix variable with $m = 3$ affine
constraints.
By the Pataki--Barvinok rank bound \cite{barvinok1995problems}, any
optimal $Y^\star$ satisfies
$r(r{+}1)/2 \leq m = 3$, giving $r \leq 1$.
Since rank-1 feasible points exist (any $\by$ in the box gives
$Y = \by\by^\top$), the optimum is attained at rank~1, i.e.,
$Y^\star = \by^\star (\by^\star)^\top$.
Thus the SDP relaxation is exact: $p^\star$ equals the true
maximum of $\sigma \cdot \by^\top Q_i \by$ over the box.

Combining: the true maximum $= p^\star$ (SDP exactness) $=$
dual optimal (strong duality) $=$ S-procedure bound.
Therefore the S-procedure certifies
$\sigma \cdot \Delta\hat{\theta}_i \leq \delta_i^{\mathrm{eff}}$
over $[-\lambda,\lambda]^2$ \emph{if and only if} the constraint
holds.
Since both the bisection and the SOS program find the largest
$\lambda$ for which all $2n$ constraints hold, they return the
same $\lambda^\star$.
\end{proof}

\section{Experiments}
\label{sec:experiments}
 
We design our experiments to answer two research questions:
 
\smallskip
\noindent
\textbf{RQ1:} \emph{Does the SOS-verified planner eliminate
joint-limit violations compared to a standard Bug2 planner, and does
this advantage persist across a range of joint-limit settings?}
 
\noindent
\textbf{RQ2:} \emph{Does preventing violations improve task-level
outcomes---goal-reaching success, path quality, and tracking
accuracy---or merely trade safety for performance?}
 
\subsection{Experimental Setup}
\label{ssec:setup}
 
\subsubsection{Task Setting}
We evaluate on a planar 3-link revolute manipulator ($n=3$) with
link lengths $\ell_1 = 1.0$, $\ell_2 = 0.8$, $\ell_3 = 0.6\,$m.
A single circular obstacle with radius $r = 0.015\,$m is placed along
the direct start--goal path to force boundary-following through
kinematically challenging regions.
The goal tolerance is $\epsilon_{\mathrm{tol}} = 0.005\,$m and the
obstacle safety margin is $0.008\,$m.
 
To ensure a comprehensive evaluation, we test six uniform per-step
joint bounds
$\delta \in \{0.020, 0.025, 0.030, 0.035, 0.040, 0.050\}\,$rad,
spanning tight to moderate constraints.
For each~$\delta$, we generate a pool of \emph{adversarial scenarios}
via an automated search (described below) and collect up to 100
scenarios per setting.
 
\subsubsection{Adversarial Scenario Generation}
\label{ssec:scenario_gen}
 
We require scenarios where vanilla Bug2 is \emph{guaranteed} to
violate joint limits while the SOS planner can succeed.
Each candidate scenario~$(\btheta_0, \bz_{\mathrm{goal}})$ is
accepted only if all of the following hold:
(i)~the starting condition number satisfies
$2.5 \leq \kappa_0 \leq 8.0$ (moderate, not degenerate);
(ii)~$\kappa(J)$ increases by a ratio $\geq 1.6\times$ along the
straight-line path (ensuring the vanilla step
$s = \delta / \kappa_0$ eventually violates);
(iii)~$\lambda^\star > 0$ at every sampled point on the path (the SOS
planner can traverse the full distance);
(iv)~the estimated step count
$\|\Delta\bz\| / (0.75 \cdot \lambda^\star_{\min}) < 500$
(the SOS planner finishes within budget);
(v)~vanilla Bug2, when run, incurs at least one joint-limit violation.
Scenarios that pass all five filters are retained.
 
\subsubsection{Planners Under Comparison}
 
\paragraph{Vanilla Bug2 (baseline).}
Fixed Cartesian step $s = \delta / \kappa_0$.
Joint displacements are computed via the first-order pseudoinverse
$\Delta\btheta = J^\dagger \Delta\bz$.
When $\|\Delta\btheta\|_\infty > \delta$, each joint increment is
\emph{clipped} to $[-\delta, \delta]$, modeling a realistic actuator
that respects physical limits.
Maximum budget: 500~steps.
 
\paragraph{SOS-Verified Bug2 (ours).}
Adaptive step $s^{(t)} = 0.75 \cdot \lambda^\star(\btheta^{(t)})$.
Joint displacements via the second-order polynomial IK~\eqref{eq:poly_ik}.
The polynomial approximation uses $\rho = 0.008$ for error estimation;
if $\lambda^\star < 10^{-6}$, $\rho$ is halved iteratively (up to 3
retries).
If a joint bound is marginally exceeded due to numerics, the step is
scaled back by a factor $0.9 \cdot \min_i (\delta_i / |\Delta\theta_i|)$.
Maximum budget: 600~steps.
 
\subsubsection{Evaluation Metrics}
 
We report the following metrics, each averaged over all scenarios
within a $\delta$~setting:
 
\begin{itemize}
  \item \textbf{Success rate} (\%): fraction of scenarios reaching
    the goal within budget.
  \item \textbf{Violation count}: number of steps where
    $\max_i |\Delta\theta_i| > \delta$ (pre-clip for vanilla,
    pre-scale-back for SOS).
  \item \textbf{Violation rate} (\%): violation count divided by
    total steps.
  \item \textbf{Final distance}: $\|\bz_{\mathrm{final}} -
    \bz_{\mathrm{goal}}\|$ at termination.
  \item \textbf{Path length ratio}: total end-effector path length
    divided by straight-line distance (1.0 is optimal).
\end{itemize}
 
\subsection{Statistical Comparison}
\label{ssec:stats}
 
Table~\ref{tab:scenario_meta} summarizes the scenario difficulty.
Across all $\delta$ settings, the starting condition number averages
$\kappa_0 \approx 3.8$ with a $\kappa$~ratio (max/start) of
$\approx 1.7\times$.
The total number of retained scenarios ranges from 9 to 22 per
$\delta$ (94 total across all settings), reflecting the stringent
five-part acceptance criterion.
 
\begin{table}[t]
  \centering
  \caption{Scenario metadata across joint-limit settings}
  \label{tab:scenario_meta}
  \begin{tabular}{@{}cccc@{}}
    \toprule
    $\delta$ (rad) & $N$ & $\kappa_0$ (mean$\pm$std) &
      $\kappa$ ratio (mean$\pm$std) \\
    \midrule
    0.020 & 22 & $3.71 \pm 0.13$ & $1.67 \pm 0.04$ \\
    0.025 & 16 & $3.79 \pm 0.16$ & $1.70 \pm 0.07$ \\
    0.030 &  9 & $3.73 \pm 0.13$ & $1.69 \pm 0.05$ \\
    0.035 & 15 & $3.80 \pm 0.13$ & $1.71 \pm 0.05$ \\
    0.040 & 11 & $3.76 \pm 0.15$ & $1.66 \pm 0.05$ \\
    0.050 & 21 & $3.77 \pm 0.12$ & $1.67 \pm 0.05$ \\
    \bottomrule
  \end{tabular}
\end{table}
 
\subsubsection{Joint-Limit Violations}
 
Table~\ref{tab:violations} presents the core safety metric.
The SOS-verified planner achieves \textbf{zero} violations in every
scenario at every $\delta$ setting, confirming the theoretical
guarantee of Theorem~\ref{thm:main}.
In contrast, vanilla Bug2 incurs $2$--$7$ violations per scenario on
average, with the violation rate increasing from $6.5\%$ at
$\delta = 0.020$ to $11.2\%$ at $\delta = 0.050$.
Larger~$\delta$ makes the fixed step $\delta/\kappa_0$ more aggressive,
amplifying the mismatch in high-$\kappa$ regions.
 
\begin{table}[t]
  \centering
  \caption{Joint-limit violations: count and rate (mean $\pm$ std, best in \textbf{bold})}
  \label{tab:violations}
  \begin{tabular}{@{}ccccc@{}}
    \toprule
    & \multicolumn{2}{c}{Violation Count} &
      \multicolumn{2}{c}{Violation Rate (\%)} \\
    \cmidrule(lr){2-3} \cmidrule(lr){4-5}
    $\delta$ & Vanilla & SOS & Vanilla & SOS \\
    \midrule
    0.020 & $2.3 \pm 1.5$ & $\mathbf{0.0 \pm 0.0}$ &
            $6.51 \pm 4.07$ & $\mathbf{0.0 \pm 0.0}$ \\
    0.025 & $2.4 \pm 1.8$ & $\mathbf{0.0 \pm 0.0}$ &
            $8.18 \pm 5.66$ & $\mathbf{0.0 \pm 0.0}$ \\
    0.030 & $3.0 \pm 1.6$ & $\mathbf{0.0 \pm 0.0}$ &
            $8.46 \pm 4.17$ & $\mathbf{0.0 \pm 0.0}$ \\
    0.035 & $3.5 \pm 3.2$ & $\mathbf{0.0 \pm 0.0}$ &
            $9.14 \pm 4.39$ & $\mathbf{0.0 \pm 0.0}$ \\
    0.040 & $5.0 \pm 5.1$ & $\mathbf{0.0 \pm 0.0}$ &
            $8.42 \pm 4.66$ & $\mathbf{0.0 \pm 0.0}$ \\
    0.050 & $6.9 \pm 5.3$ & $\mathbf{0.0 \pm 0.0}$ &
            $11.24 \pm 5.51$ & $\mathbf{0.0 \pm 0.0}$ \\
    \bottomrule
  \end{tabular}
\end{table}
 
\subsubsection{Goal-Reaching Success and Path Quality}
 
Table~\ref{tab:success_path} reports success rate, final distance,
and path efficiency.
The SOS planner achieves \textbf{100\% success} at all $\delta$
settings.
Vanilla Bug2 reaches the goal reliably at tight limits
($\delta \leq 0.030$) because the small fixed step happens to be
close to the safe threshold, but at looser limits
($\delta \geq 0.035$) its success rate drops to 82--93\% as the
larger step induces more clipping drift.
 
The path length ratio reveals a striking contrast.
At $\delta = 0.020$--$0.025$ both planners follow near-optimal paths
(${\approx}1.2\times$ straight-line).
At $\delta = 0.040$--$0.050$, vanilla Bug2's drift inflates the path
to $8.8$--$10.3\times$ straight-line, while the SOS planner maintains
a ratio of ${\approx}1.2$--$1.5\times$.
 
\begin{table}[t]
  \centering
  \caption{Success rate, final distance, and path length ratio
    (mean~$\pm$~std; best in \textbf{bold})}
  \label{tab:success_path}
  \begin{tabular}{@{}ccccc@{}}
    \toprule
    & \multicolumn{2}{c}{Success Rate (\%)} &
      \multicolumn{2}{c}{Path Length Ratio} \\
    \cmidrule(lr){2-3} \cmidrule(lr){4-5}
    $\delta$ & Vanilla & SOS & Vanilla & SOS \\
    \midrule
    0.020 & 100.0 & \textbf{100.0} &
            $1.21 \pm 0.02$ & $\mathbf{1.17 \pm 0.02}$ \\
    0.025 & 100.0 & \textbf{100.0} &
            $1.23 \pm 0.02$ & $\mathbf{1.18 \pm 0.02}$ \\
    0.030 & 100.0 & \textbf{100.0} &
            $1.94 \pm 1.05$ & $\mathbf{1.20 \pm 0.01}$ \\
    0.035 &  93.3 & \textbf{100.0} &
            $3.85 \pm 7.48$ & $\mathbf{1.21 \pm 0.02}$ \\
    0.040 &  81.8 & \textbf{100.0} &
            $8.85 \pm 13.0$ & $\mathbf{1.22 \pm 0.03}$ \\
    0.050 &  85.7 & \textbf{100.0} &
            $10.3 \pm 15.2$ & $\mathbf{1.47 \pm 0.67}$ \\
    \bottomrule
  \end{tabular}
\end{table}
 
\subsubsection{Computational Cost}
 
The SOS-verified planner averages $0.009$--$0.018\,$s per scenario
(Table~\ref{tab:time}), corresponding to
${\approx}0.3\,$ms per planning step using the bisection procedure.
Vanilla Bug2 is faster ($<0.001\,$s) since it performs only a
matrix--vector multiply per step, but the SOS planner's
sub-millisecond overhead is well within real-time requirements for
typical manipulation tasks.
 
\begin{table}[t]
  \centering
  \caption{Computational cost and step count (best in \textbf{bold})}
  \label{tab:time}
  \begin{tabular}{@{}ccccc@{}}
    \toprule
    & \multicolumn{2}{c}{Steps} &
      \multicolumn{2}{c}{Wall Time (s)} \\
    \cmidrule(lr){2-3} \cmidrule(lr){4-5}
    $\delta$ & Vanilla & SOS & Vanilla & SOS \\
    \midrule
    0.020 & $\mathbf{35.3}$ & $52.4$ & $\mathbf{{<}0.001}$ & $0.018$ \\
    0.025 & $\mathbf{29.4}$ & $43.3$ & $\mathbf{{<}0.001}$ & $0.014$ \\
    0.030 & $38.3$ & $\mathbf{36.4}$ & $\mathbf{{<}0.001}$ & $0.012$ \\
    0.035 & $64.9$ & $\mathbf{33.1}$ & $\mathbf{{<}0.001}$ & $0.011$ \\
    0.040 & $126.8$ & $\mathbf{28.1}$ & $\mathbf{{<}0.001}$ & $0.010$ \\
    0.050 & $113.8$ & $\mathbf{27.1}$ & $\mathbf{{<}0.001}$ & $0.009$ \\
    \bottomrule
  \end{tabular}
\end{table}
 
An important observation is that for $\delta \geq 0.030$, the SOS
planner actually uses \emph{fewer steps} than vanilla Bug2.
This is because vanilla's clipping-induced drift causes it to wander
off the intended path, requiring many corrective steps, whereas the
SOS planner follows a direct, drift-free trajectory.
 
\subsection{Qualitative Analysis}
\label{ssec:qualitative}
 
Fig.~\ref{fig:comparison} illustrates a representative scenario with
$\delta = 0.035\,$rad, $\kappa_0 = 6.4$, and
$\kappa_{\max} = 34.2$ ($5.3\times$ ratio). The path difference visible in Fig.~\ref{fig:comparison} is \emph{not}
due to randomness or backtracking---Bug2 is a deterministic algorithm.
The cause is clipping-induced drift.
When the vanilla planner's pseudoinverse IK requests a joint
displacement that exceeds $\delta$, the displacement is clipped to
$[-\delta, \delta]$.
The actual end-effector position after clipping differs from what the
planner intended, so at the next step the planner computes the
goal direction from a \emph{different} current position.
This drift accumulates: each clipped step shifts the starting point
of the next, producing a trajectory that progressively deviates
from the intended path.
The SOS-verified planner, by contrast, never clips---every step
lands exactly where intended---so it follows the straight
go-to-goal / boundary-follow path that Bug2 would follow in an
ideal (unconstrained) setting.
 
\paragraph{Cartesian paths (top row).}
The vanilla planner's path (top-left) shows red segments where joint
limits were violated and subsequently clipped, producing visible kinks.
The SOS planner (top-right) follows a smooth path; the purple
rectangles show the $\lambda^\star$-boxes, which visibly shrink near
the obstacle where $\kappa(J)$ peaks.
 
\paragraph{Joint-step compliance (middle-left).}
The vanilla planner's pre-clip joint deviations (red) exceed the
dashed $\delta$-line at three points.
The SOS planner's deviations (green) remain strictly below the limit
throughout.
 
\paragraph{Adaptive $\lambda^\star$ (middle-right).}
The purple curve shows $\lambda^\star$ dropping from $0.008$ to
$0.005$ as $\kappa(J)$ (orange dashed) rises from 6 to 12, then
recovering.
This adaptation is the mechanism by which the planner avoids
violations without any explicit joint-limit checking in the planning
loop.
 
\paragraph{Arm configurations (bottom-right).}
Sampled arm poses along the SOS path show smooth, continuous motion
with no abrupt reconfigurations.
 
\begin{figure}[t]
  \centering
  \includegraphics[width=\columnwidth]{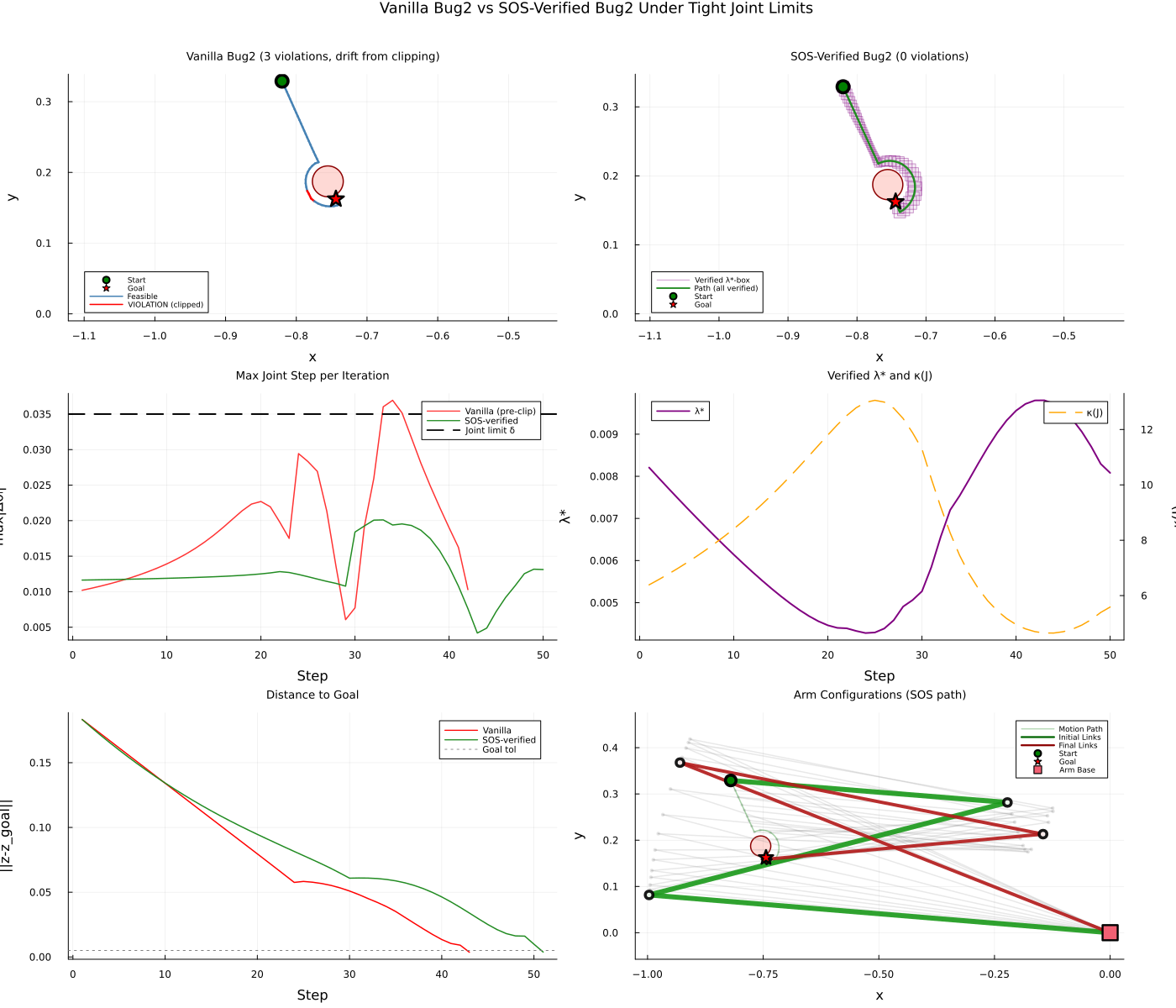}
  \caption{%
  Comparison of Vanilla Bug2 (left) and SOS-Verified Bug2 (right).
    \textbf{Top:} Cartesian paths; red segments mark steps where
    joint limits were violated and clipped (vanilla only), causing
    the trajectory to deviate from the SOS path.
    Purple boxes show the certified $\lambda^\star$-regions.
    The two planners follow \emph{different} paths because clipping
    shifts vanilla's current position at each violated step,
    altering subsequent Bug2 direction decisions
    (see text for details).
    \textbf{Middle-left:} Max joint step per iteration; the dashed
    line is $\delta = 0.035$.
    \textbf{Middle-right:} $\lambda^\star$ (purple) and $\kappa(J)$
    (orange); note the inverse relationship.
    \textbf{Bottom:} Distance to goal and sampled arm configurations.
  }
  \label{fig:comparison}
\end{figure}

\subsection{Discussion}
 
\paragraph{Answering RQ1.}
The SOS-verified planner achieves zero joint-limit violations across
\emph{all} 94 adversarial scenarios and \emph{all} six $\delta$
settings, while vanilla Bug2 violates limits in 6--11\% of steps.
The advantage is consistent and grows with~$\delta$.
 
\paragraph{Answering RQ2.}
Eliminating violations does not sacrifice performance; it improves it.
At $\delta \geq 0.035$, the SOS planner achieves higher success rates
(100\% vs.\ 82--93\%), shorter paths ($1.2\times$ vs.\
$4$--$10\times$ straight-line), and fewer steps (27--33 vs.\ 65--127).
The only cost is a modest computational overhead of ${\approx}0.3\,$ms
per step, which is negligible for typical planning rates.

\section{Conclusion}
\label{sec:conclusion}
 
We have presented an approach to task-space reactive motion planning
that provides formal joint-feasibility guarantees at every planning
step.
By computing a certified Cartesian reachable set via SOS programming
and using it to adaptively size the steps of a Bug2 planner, we
eliminate the joint-limit violations and tracking drift that arise
from fixed-step approaches.
A statistical evaluation across 94 adversarial scenarios and six
joint-limit settings confirms the theoretical guarantees:
the SOS-verified planner achieves zero joint violations and 100\%
goal reaching in all cases, whereas vanilla Bug2 violates limits in
6--11\% of steps and fails to reach the goal in up to 18\% of
scenarios at looser joint bounds.
Crucially, eliminating violations does not degrade performance; it
improves path quality and reduces step count by avoiding
clipping-induced drift.
Future work will extend the evaluation to higher-DOF systems,
integrate with sampling-based planners (e.g., RRT) for complex
3D environments, and explore online adaptation of the polynomial
approximation order.


\bibliographystyle{IEEEtran}
\bibliography{arXiv}
\end{document}